  \documentclass[10pt]{article}
  \usepackage{amsmath,amsthm,amsfonts,amssymb,latexsym,graphicx,stmaryrd}
  \usepackage{cite}
  \usepackage{algorithm}
  \usepackage{algpseudocode}
  \usepackage[pdfpagemode=UseNone,pdfstartview=FitH]{hyperref}
  \newcommand{\Extra}[1]{}

\allowdisplaybreaks[4]

  \newcommand{\OCMXVII}{Vovk/etal:2019ML}
  \newcommand{\OCMXVIII}{Vovk:arXiv1708}
  
  \newcommand{\OCMXX}{Vovk/etal:2018Braverman}
  
  \newcommand{\OCMXXII}{Vovk/etal:2018COPA}

  \newtheorem{theorem}{Theorem}
  
  \newtheorem{corollary}[theorem]{Corollary}
  \newtheorem{proposition}[theorem]{Proposition}
  \theoremstyle{definition}
  \newtheorem{definition}{Definition}

\DeclareMathOperator{\Prob}{\mathbb{P}}

\DeclareMathOperator{\conv}{conv}

\newcommand{\R}{\mathbb{R}}
\newcommand{\G}{\mathbb{G}}
\newcommand{\U}{\mathbb{U}}

  \makeatletter
    \def\citet{\@ifnextchar[{\@twith}{\@twithout}}
    \def\@twith[#1]#2{\cite[#1]{#2}}
    \def\@twithout#1{\cite{#1}}
    \def\citep{\@ifnextchar[{\@pwith}{\@pwithout}}
    \def\@pwith[#1]#2{\cite[#1]{#2}}
    \def\@pwithout#1{\cite{#1}}
    \def\citealt{\@ifnextchar[{\@altwith}{\@altwithout}}
    \def\@altwith[#1]#2{\cite[#1]{#2}}
    \def\@altwithout#1{\cite{#1}}
    \def\citealp{\@ifnextchar[{\@alpwith}{\@alpwithout}}
    \def\@alpwith[#1]#2{\cite[#1]{#2}}
    \def\@alpwithout#1{\cite{#1}}
  \makeatother
  \title{Conformal calibrators}
  \author{Vladimir Vovk, Ivan Petej, Paolo Toccaceli, and Alex Gammerman}

\begin{document}
\maketitle

\begin{abstract}
  Most existing examples of full conformal predictive systems,
  split-conformal predictive systems,
  and cross-conformal predictive systems
  impose severe restrictions on the adaptation of predictive distributions
  to the test object at hand.
  In this paper we develop split-conformal and cross-conformal predictive systems
  that are fully adaptive.
  Our method consists in calibrating existing predictive systems;
  the input predictive system is not supposed to satisfy any properties of validity,
  whereas the output predictive system is guaranteed to be calibrated in probability.
  It is interesting that the method may also work
  without the IID assumption, standard in conformal prediction.

     The version of this paper at \url{http://alrw.net} (Working Paper 23)
     is updated most often.
\end{abstract}

\section{Introduction}

Conformal predictive distributions were inspired by the work on predictive distributions in parametric statistics
(see, e.g., \citealp[Chapter~12]{Schweder/Hjort:2016} and \citealp{Shen/etal:2017})
and first suggested in \citet{\OCMXVII}.
As usual, we will refer to algorithms producing conformal predictive distributions
as conformal predictive systems
(CPS, used in both singular and plural senses).

Conformal predictive systems are built on top of traditional prediction algorithms
to ensure a property of validity usually referred to as calibration in probability
\citep{Gneiting/Katzfuss:2014}.
Several versions of the Least Squares Prediction Machine,
CPS based on the method of Least Squares,
are constructed in \citet{\OCMXVII}.
This construction is slightly extended to cover ridge regression
and then further extended to nonlinear settings by applying the kernel trick
in \citet{\OCMXX}.
However, even after this extension the method is not fully adaptive,
even for a universal kernel.
As explained in \citet[Section~7]{\OCMXX},
the universality of the kernel shows in the ability of the predictive distribution function
to take any shape;
however, the CPS is still inflexible in that the shape does not depend,
or depends weakly, on the test object.

For many base algorithms full CPS
(like full conformal predictors in general)
are computationally inefficient,
and \citet{\OCMXXII} define and study computationally efficient versions of CPS,
namely split-conformal predictive systems (SCPS)
and cross-conformal predictive systems (CCPS).
However, specific SCPS and CCPS proposed in \citet{\OCMXXII}
are based on the split conformity measure
\begin{equation}\label{eq:example}
  A(z_1,\dots,z_m,(x,y))
  :=
  \frac{y-\hat y}{\hat\sigma},
\end{equation}
where $\hat y$ is a prediction for $y$ computed from $x$ as test object
and $z_1,\dots,z_m$ as training sequence,
and $\hat\sigma>0$ is an estimate of the quality of $\hat y$ computed from the same data.
The predictive distributions corresponding to \eqref{eq:example}
are slightly more adaptive:
not only their location but also their scale depends on the test object $x$.
Ideally, however, we would like to allow
a stronger dependence on the test object.
This paper follows \citet[Section~10]{\OCMXVIII}
in using a method that is fully flexible
and, for a suitable base algorithm,
adapts fully to the test object,
both asymptotically and in practical problems
(cf.\ Proposition~\ref{prop:adapt} below).
Whereas the emphasis in \citet{\OCMXVIII}
is on asymptotic optimality only,
one of the purposes of this paper is to propose practically useful solutions.

This is a very preliminary version of the paper;
we plan to submit a more mature and self-contained version to COPA 2019.
For now we will freely use the terminology and notation of \citet{\OCMXXII}.

\section{Predictive systems and randomized predictive systems}

Let us fix (until Section~\ref{sec:experiments})
a nonempty measurable space $\mathbf{X}$,
which will serve as our \emph{object space},
and let $\mathbf{Z}:=\mathbf{X}\times\R$ stand for our \emph{observation space}.
Each observation $z=(x,y)\in\mathbf{Z}$ consists of an object $x\in\mathbf{X}$ and its label $y\in\R$.

\begin{definition}
  A measurable function $Q:\mathbf{Z}^{n+1}\to[0,1]$ is called a \emph{predictive system} (PS) if:
  \begin{enumerate}
  \item
    For each training sequence $(z_1,\dots,z_n)\in\mathbf{Z}^n$ and each test object $x\in\mathbf{X}$,
    the function $Q(z_1,\dots,z_n,(x,y))$ is monotonically increasing in $y$
    (where ``monotonically increasing'' is understood in the wide sense allowing intervals of constancy).
  \item
    For each training sequence $(z_1,\dots,z_n)\in\mathbf{Z}^n$ and each test object $x\in\mathbf{X}$,
    \begin{equation*}
      \lim_{y\to-\infty}
      Q(z_1,\dots,z_n,(x,y))
      =
      0
    \end{equation*}
    and
    \begin{equation*}
      \lim_{y\to\infty}
      Q(z_1,\dots,z_n,(x,y))
      =
      1.
    \end{equation*}
  \end{enumerate}
\end{definition}
The output $y\in\R\mapsto Q(z_1,\dots,z_n,(x,y))$ of a PS on a given training sequence $z_1,\dots,z_n$
and test object $x$ will be referred to as a \emph{predictive distribution}
and will sometimes be denoted $Q_{z_1,\dots,z_n,x}$.
It is a distribution function in the sense of probability theory
except that we do not require that it be right-continuous.

We also need the notion of a randomized predictive system.
\begin{definition}\label{def:RPS}
  A measurable function $Q:\mathbf{Z}^{n+1}\times[0,1]\to[0,1]$ is called a \emph{randomized predictive system} (RPS) if:
  \begin{enumerate}
  \item
    For each training sequence $(z_1,\dots,z_n)\in\mathbf{Z}^n$ and each test object $x\in\mathbf{X}$,
    the function $Q(z_1,\dots,z_n,(x,y),\tau)$ is monotonically increasing in $y$
    and monotonically increasing in $\tau$.
  \item
    For each training sequence $(z_1,\dots,z_n)\in\mathbf{Z}^n$ and each test object $x\in\mathbf{X}$,
    \begin{equation*}
      \lim_{y\to-\infty}
      Q(z_1,\dots,z_n,(x,y),0)
      =
      0
    \end{equation*}
    and
    \begin{equation*}
      \lim_{y\to\infty}
      Q(z_1,\dots,z_n,(x,y),1)
      =
      1.
    \end{equation*}
  \end{enumerate}
\end{definition}
The output $y\in\R\mapsto Q(z_1,\dots,z_n,(x,y),\tau)$ of an RPS
on a given training sequence $z_1,\dots,z_n$, test object $x$, and (random) number $\tau$
will be referred to as a \emph{predictive distribution} (function)
and will sometimes be denoted $Q_{z_1,\dots,z_n,x,\tau}$.

Notice that Definition~\ref{def:RPS} does not include any requirement of validity,
unlike the corresponding definitions in \citet{\OCMXVII,\OCMXVIII,\OCMXX,\OCMXXII}:
in this paper we follow the terminology of \citet[Chapter~12]{Schweder/Hjort:2016}
rather than \citet{Shen/etal:2017}.

An RPS $Q$ is \emph{calibrated in probability} if,
for any probability measure $P$ on $\mathbf{Z}$,
as function of random training observations $Z_1\sim P$,\dots, $Z_n\sim P$,
a random test observation $Z\sim P$, and a random number $\tau\sim U$
($U$ being the uniform probability measure on $[0,1]$),
all assumed independent,
the distribution of $Q$ is uniform:
\begin{equation}\label{eq:R2}
  \forall \alpha\in[0,1]:
  \Prob
  \left(
    Q(Z_1,\dots,Z_n,Z,\tau)
    \le
    \alpha
  \right)
  =
  \alpha.
\end{equation}
(This was included as Requirement~R2 in the definition of an RPS
in \citealt{\OCMXVII,\OCMXVIII,\OCMXX,\OCMXXII}.)

\section{Split-conformal calibrators}

When considered as a split conformity measure
each predictive system is balanced and isotonic
(at least if we ignore its values 0 and 1),
which makes it possible to apply Proposition~3.1 in \citet{\OCMXXII}.

If $A$ is a predictive system,
the \emph{split-conformalized predictive system} (SCPS)
corresponding to $A$ is defined as follows
(following the definition of a split-conformal transducer in \citealt{\OCMXXII}).
The training sequence $z_1,\dots,z_n$ is split into two parts:
the \emph{training sequence proper} $z_1,\dots,z_m$ and the \emph{calibration sequence} $z_{m+1},\dots,z_n$;
we are given a test object $x$.
The output of $C^A$ is defined as
\begin{multline}\label{eq:SCPS}
  C^A(z_1,\dots,z_n,(x,y),\tau)
  :=
  \frac{1}{n-m+1}
  \left|\left\{i=m+1,\dots,n\mid\alpha_i<\alpha^y\right\}\right|\\
  +
  \frac{\tau}{n-m+1}
  \left|\left\{i=m+1,\dots,n\mid\alpha_i=\alpha^y\right\}\right|
  +
  \frac{\tau}{n-m+1},
\end{multline}
where the \emph{conformity scores} $\alpha_i$, $i=m+1,\dots,n$, and $\alpha^y$, $y\in\R$,
are defined by
\begin{equation*}
  \begin{aligned}
    \alpha_i
    &:=
    A(z_1,\dots,z_m,(x_i,y_i)),
      \qquad i=m+1,\dots,n,\\
    \alpha^y
    &:=
    A(z_1,\dots,z_m,(x,y)).
  \end{aligned}
\end{equation*}

For simplicity, let us assume that $A$ never takes values $0$ and $1$.
By \citet[Proposition~3.1]{\OCMXXII},
every split-conformalized predictive system is an RPS.
The functional mapping predictive systems to the corresponding split-conformalized predictive systems
are \emph{split-conformal calibrators}.

  \begin{algorithm}[bt]
    \caption{Split-Conformal Calibrator}
    \label{alg:SCPS}
    \begin{algorithmic}
      \Require
        A training sequence $(x_i,y_i)\in\mathbf{Z}$, $i=1,\dots,n$, and $m<n$.
      \Require
        A test object $x\in\mathbf{X}$ and random number $\tau\in[0,1]$.
      \For{$i\in\{1,\dots,n-m\}$}
        \State $p_i:=A(z_1,\dots,z_m,z_{m+i})$
      \EndFor
      \State sort $p_1,\dots,p_{n-m}$ in the increasing order obtaining $p_{(1)}<\dots<p_{(k)}$
      \For{$j\in\{1,\dots,k\}$}
        \State $n_j:=\left|\left\{i=1,\dots,n-m\mid p_i=p_{(j)}\right\}\right|$
        \State $m_j:=\sup\{y\mid A(z_1,\dots,z_m,(x,y))<p_{(j)}\}$
        \State $M_j:=\inf\{y\mid A(z_1,\dots,z_m,(x,y))>p_{(j)}\}$
      \EndFor
      \State return the predictive distribution $C^A$ given by~\eqref{eq:C} for the label $y$ of $x$.
    \end{algorithmic}
  \end{algorithm}

The SCPS $C^A$ can be implemented by directly coding the definition \eqref{eq:SCPS}
using a grid of values of $y$
(as we do for the experiments in Section~\ref{sec:experiments}).
Algorithm~\ref{alg:SCPS} describes another implementation of $C^A$.
It defines the predictive distribution apart from a finite number of points $y$
(and so the values at those points do not affect, e.g., CRPS);
we can set the probability interval
$\conv(\{C^A(z_1,\ldots,z_n,(x,y),\tau)\mid\tau\in[0,1]\})$
at those points $y$ to the union of the probability intervals at the adjacent points
without a substantial change to the predictive system.
Some of the $p_i$, $i=1,\dots,n-m$, may coincide,
so we can only say that $k\in\{1,\dots,n-m\}$
(notice that the sequence $p_{(j)}$, $j=1,\dots,k$, is strictly increasing).
The predictive distribution that it outputs is
\begin{multline}\label{eq:C}
  C^A(z_1,\dots,z_n,(x,y),\tau)
  ={}\\
  \begin{cases}
    \frac{\tau}{n-m+1} & \text{if $y<m_1$}\\
    \frac{n_1+\dots+n_{j-1}+\tau n_j+\tau}{n-m+1} & \text{if $y\in(m_j,M_j)$, $j\in\{1,\dots,k\}$}\\
    \frac{n_1+\dots+n_j+\tau}{n-m+1} & \text{if $y\in(M_j,m_{j+1})$, $j\in\{1,\dots,k-1\}$}\\
    \frac{n_1+\dots+n_k+\tau}{n-m+1} = \frac{n-m+\tau}{n-m+1} & \text{if $y>M_k$}.
  \end{cases}
\end{multline}

Algorithm~\ref{alg:SCPS} is a slight generalization of Algorithm~1 in \citet{\OCMXXII}.
The latter makes an assumption
(the base distribution functions $A_{z_1,\dots,z_n,x}$ being continuous and strictly increasing)
implying that $m_j=M_j$ for all $j\in\{1,\dots,k\}$;
in our current general context we can only say that
\[
  m_1 \le M_1 \le m_2 \le M_2 \le \dots \le m_k \le M_k.
\]

Notice that the split conformity measure \eqref{eq:example},
which is used in \citet{\OCMXXII},
is not covered directly by our definition since it does not have to take values in $[0,1]$.
But this can be easily arranged:
e.g., we can apply the sigmoid function to \eqref{eq:example}
to make sure it takes values in $[0,1]$.

Split-conformal predictive systems are automatically calibrated in probability,
in the sense of satisfying~\eqref{eq:R2},
under the IID assumption.
If $F$ is the distribution function produced for a test object $X^*$,
\(
  F
  :=
  C^A_{Z_1,\dots,Z_n,X^*,\tau}
\),
then $F(Y^*)$ will be distributed uniformly on $[0,1]$,
where $Y^*$ is the true label of $X^*$.
Notice, however, that for a test sequence $Z^*_i=(X^*_i,Y^*_i)$,
$i=1,\dots,k$,
$F_i(Y_i^*)$ will not be independent,
even though distributed uniformly on $[0,1]$,
where $F_i:=C^A_{Z_1,\dots,Z_n,X^*_i,\tau_i}$
is the distribution function produced for $X^*_i$.
To make $F_i(Y_i^*)$ not only distributed uniformly on $[0,1]$
but also independent,
we can use the ``semi-online'' protocol,
predicting the labels $Y^*_i$ of $X^*_i$, $i=1,\dots,k$, sequentially
and adding $Z^*_i$ to the calibration sequence as soon as it is processed.
(This assumes that $Z_i$, $Z^*_i$, and $\tau_i$ are all independent.)
This remark might be useful for debugging implementations of split-conformal calibrators.

\section{Cross-conformal calibrators}

We can easily combine several split-conformal calibrators
into a cross-conformal calibrator,
exactly in the same way as in \citet[Section~4]{\OCMXXII}.
The resulting RPS will lose automatic calibration in probability~\eqref{eq:R2}
but will use the available data more efficiently.

\section{Conformalizing ideal predictive systems}

In this section we will explore the efficiency of conformal calibrators
in the situation where the base predictive system $A$ is the ideal one.
In this case we cannot improve $A$,
and we are interested in how much worse $C^A$ can become as compared with $A$.
(This is the question asked in a slightly different context
independently by Evgeny Burnaev and Larry Wasserman.)
If, for any $A$, $C^A$ is almost as good as $A$,
we can say that our conformal calibrator is fully adaptive.

In this section we only consider the IID case.
Let $P$ be the true probability measure on $\mathbf{Z}$
generating the observations $z_1,z_2,\dots$.
A \emph{conditional distribution function} for $P$
is a right-continuous function $A:\mathbf{Z}\to[0,1]$ satisfying, for each $y\in\R$,
\begin{equation}\label{eq:cdf}
  A(X,y)
  =
  \Prob(Y\le y\mid X)
  \quad
  \text{a.s.}
\end{equation}
when $(X,Y)\sim P$.
The existence and a.s.\ uniqueness of a conditional distribution function
follows from standard results about the existence of regular probability distributions
(e.g., \citealt[Theorem~10.2.2]{Dudley:2002}).

Consider a sequence $\xi_1,\xi_2,\dots$ of independent and uniformly distributed random variables
$\xi_i\sim U$.
Let $\G_n$ be the empirical distribution function of $\xi_1,\dots,\xi_n$;
we are using the notation of \citet{Shorack/Wellner:1986},
who refer to $\G_n$ as the \emph{uniform empirical distribution function}.
For large $n$ and with high probability,
$\G_n$ is close to the main diagonal of the unit square $[0,1]^2$.

Let us use the true conditional distribution function $A$ as base predictive system
(roughly, this corresponds to an infinitely long training sequence proper, $m=\infty$).
The corresponding \emph{ideal conformalized predictive system} (ICPS)
is defined as
\begin{multline*}
  C^A(z_1,\dots,z_n,(x,y),\tau)
  :=
  \frac{1}{n+1}
  \left|\left\{i=1,\dots,n\mid A(x_i,y_i)<A(x,y)\right\}\right|\\
  +
  \frac{\tau}{n+1}
  \left|\left\{i=1,\dots,n\mid A(x_i,y_i)=A(x,y)\right\}\right|
  +
  \frac{\tau}{n+1},
\end{multline*}
where $x$ is the test object.
Intuitively, the whole training sequence is used as the calibration sequence
(we do not need a training sequence proper as $A$ is already perfect).
An ICPS is an idealization of both SCPS and CCPS.

The following two propositions say that $C^A$ will be close to $A$
and that the distance between them will be of order $n^{-1/2}$.

\begin{proposition}\label{prop:adapt}
  Suppose the conditional distribution function $A_x:=A(x,\cdot)$
  (for the true probability measure)
  is continuous and strictly increasing
  for almost all $x\in\mathbf{X}$.
  Then the ICPS $C^A$ satisfies
  \begin{equation*}
    \left(
      C^A_{Z_1,\dots,Z_n,X,\tau}
      \circ
      A^{-1}_{X}
    \right)_{n=1}^{\infty}
    \stackrel d=
    \left(
      \G_n + \eta_n
    \right)_{n=1}^{\infty},
  \end{equation*}
  where $\stackrel d=$ means the equality of distributions
  and $\eta_n$ are random functions in the Skorokhod space $D[0,1]$
  satisfying $\left\|\eta_n\right\|_{\infty}\le1/(n+1)$ a.s.
\end{proposition}

\begin{proof}
  For given $t\in[0,1]$ and $n$,
  \begin{multline*}
    C^A_{Z_1,\dots,Z_n,X,\tau}
    \left(
      A^{-1}_{X}(t)
    \right)
    =
    \frac{1}{n+1}
    \left|\left\{
      i\in\{1,\dots,n\}
      \mid
      A_{X_{i}}(Y_i) < t
    \right\}\right|\\
    +
    \frac{\tau}{n+1}
    \left|\left\{
      i\in\{1,\dots,n\}
      \mid
      A_{X_{i}}(Y_i) = t
    \right\}\right|
    +
    \frac{\tau}{n+1}
    =
    \frac{k}{n+1}
    +
    \frac{\tau}{n+1},
  \end{multline*}
  where the second equality holds almost surely and
  \[
    k
    :=
    \left|\left\{
      i\in\{1,\dots,n\}
      \mid
      A_{X_{i}}(Y_i) \le t
    \right\}\right|.
  \]
  It remains to notice that the probability integral transforms $A_{X_{i}}(Y_i)\sim U$ are IID
  and that
  \[
    \sup_{\tau\in[0,1],k\in\{0,\dots,n\}}
    \left|
      \frac{k}{n+1}
      +
      \frac{\tau}{n+1}
      -
      \frac{k}{n}
    \right|
    =
    \frac{1}{n+1}.
    \qedhere
  \]
\end{proof}

\begin{corollary}\label{cor:adapt}
  Suppose the conditional distribution function $A_x$
  is continuous and strictly increasing for almost all $x\in\mathbf{X}$.
  Then the ICPS $C^A$ corresponding to $A$
  approaches $A$ in the sense of
  \begin{equation}\label{eq:adapt}
    \sqrt{n}
    \left(
      C^A_{Z_1,\dots,Z_n,X,\tau}
      \circ
      A^{-1}_{Z_1,\dots,Z_n,X}
      -
      I
    \right)
    \Rightarrow
    \U
    \quad
    \text{as $n\to\infty$},
  \end{equation}
  where $I:[0,1]\to[0,1]$ is the identity function $I(t)=t$, $t\in[0,1]$,
  and $\U$ is a Brownian bridge.
\end{corollary}

\begin{proof}
  This follows from Proposition~\ref{prop:adapt}
  by the standard result $n^{1/2}(\G_n-I)\Rightarrow\U$
  about the weak convergence of empirical processes
  (see, e.g., \citealt[Theorem~16.4]{Billingsley:1968})
  and the invariance of weak convergence under small perturbations
  (e.g., \citealt[Theorem~4.1]{Billingsley:1968}).
\end{proof}

According to \eqref{eq:adapt},
the speed of convergence of $C^A$ to $A$ is indeed $O(n^{-1/2})$.
This speed of convergence is the same as for the Dempster--Hill procedure
\citealp[Section~5.1]{\OCMXVII}.
In the case of Gaussian $y_i$ (and with $x_i$ absent),
this is stated in \citet[Theorem~4]{\OCMXVII},
but it is true without any parametric assumptions.
Notice that the Dempster--Hill procedure is a special case of our procedure
corresponding to $x_i$ absent and any continuous and strictly increasing
\[
  A_{z_1,\dots,z_m,x}
  =
  A_{y_1,\dots,y_m}
  =
  A
\]
(the first equality saying that $x_i$ are absent and the second being our restriction on $A$).
Since conformity measures $A$ and $\phi(A)$ lead to the same conformal transducer
provided $\phi$ is strictly increasing,
we can just set $A(y):=y$, $y\in\R$.

\section{Experimental results}
\label{sec:experiments}

The main question that we plan to explore in this section
is whether our conformalization procedure improves the performance of standard predictive systems
for artificial and benchmark data sets.
(Alternatively, it might happen that standard predictive systems are calibrated or almost calibrated automatically,
and the extra calibration step does not help.)
In this version of the paper we only consider one standard predictive system
and one toy artificial data set.

The predictive system that we consider is the Nadaraya--Watson predictive system
(first introduced in the density form in \citealt{Rosenblatt:1969})
\begin{equation}\label{eq:NW}
  F(y\mid x)
  =
  \frac
  {
    \sum_{i=1}^n
    \sigma
    \left(
      \frac{y-y_i}{h}
    \right)
    G
    \left(
      \frac{x-x_i}{g}
    \right)
  }
  {
    \sum_{i=1}^n
    G
    \left(
      \frac{x-x_i}{g}
    \right)
  },
\end{equation}
where we will take $\sigma$ to be the sigmoid distribution function
\[
  \sigma(u)
  :=
  \frac{1}{1+e^{-u}}
\]
and $G$ the Gaussian kernel
\[
  G(u)
  :=
  \frac{1}{\sqrt{2\pi}}
  e^{-u^2/2}.
\]

\begin{figure}[bt]
  \begin{center}
    \includegraphics[width=0.7\textwidth,trim={9mm 8mm 9mm 28mm},clip]{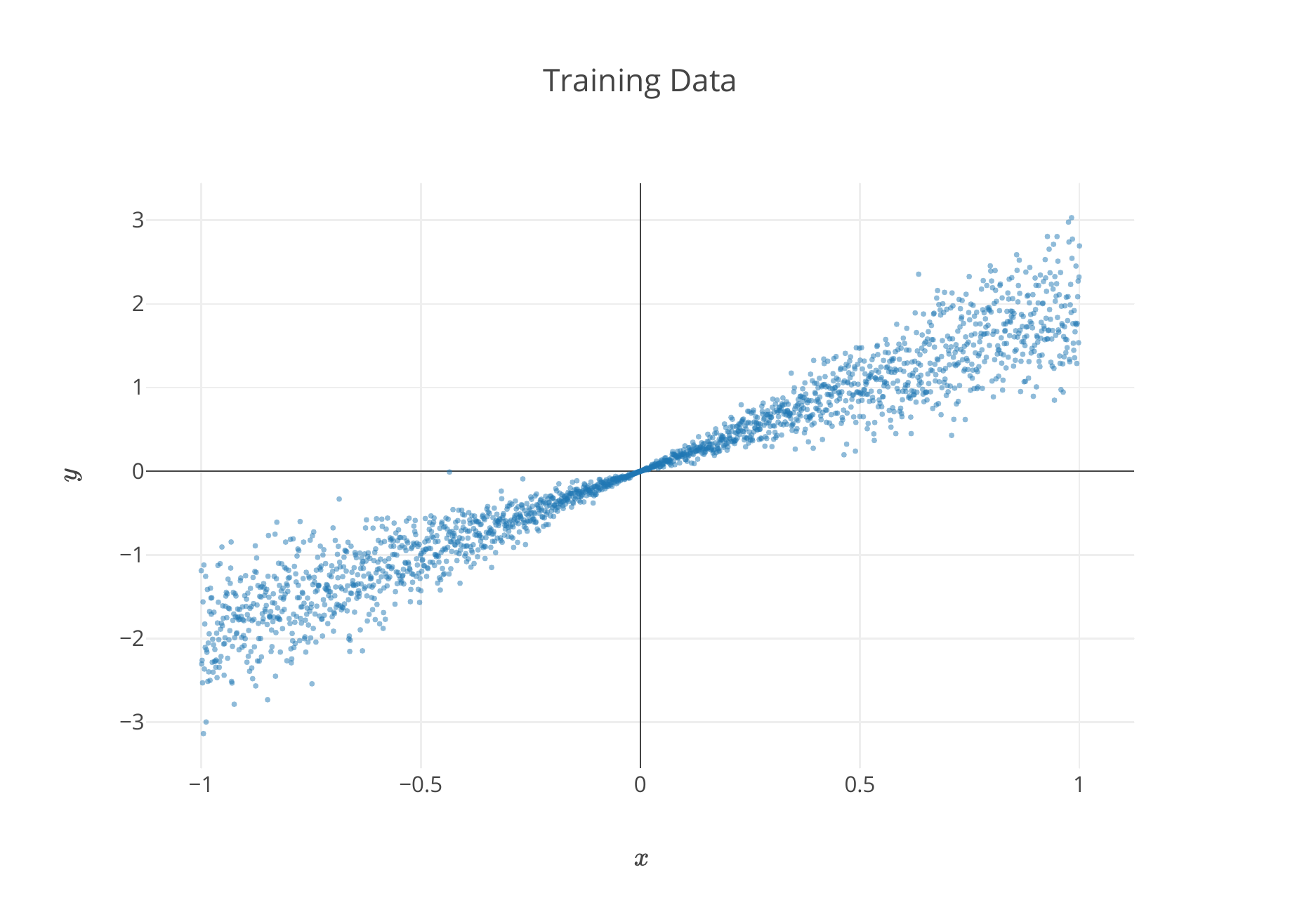}
  \end{center}
  \caption{The toy training set.}
  \label{fig:data}
\end{figure}

The labels $y_i$ are generated as
\[
  y_i := 2x_i + \epsilon_i,
\]
where $\epsilon_i$ is Gaussian noise with mean 0 and standard deviation $\left|x_i\right|/2$,
and the objects $x_i$ are drawn from the uniform distribution on $[-1,1]$;
$x_i$ and $\epsilon_i$, $i=1,2,\dots$, are all independent.
A training set of size $2000$ is shown in Figure~\ref{fig:data}.

\begin{figure}[bt]
  \begin{center}
    \includegraphics[width=\textwidth,trim={9mm 9mm 7mm 26mm},clip]{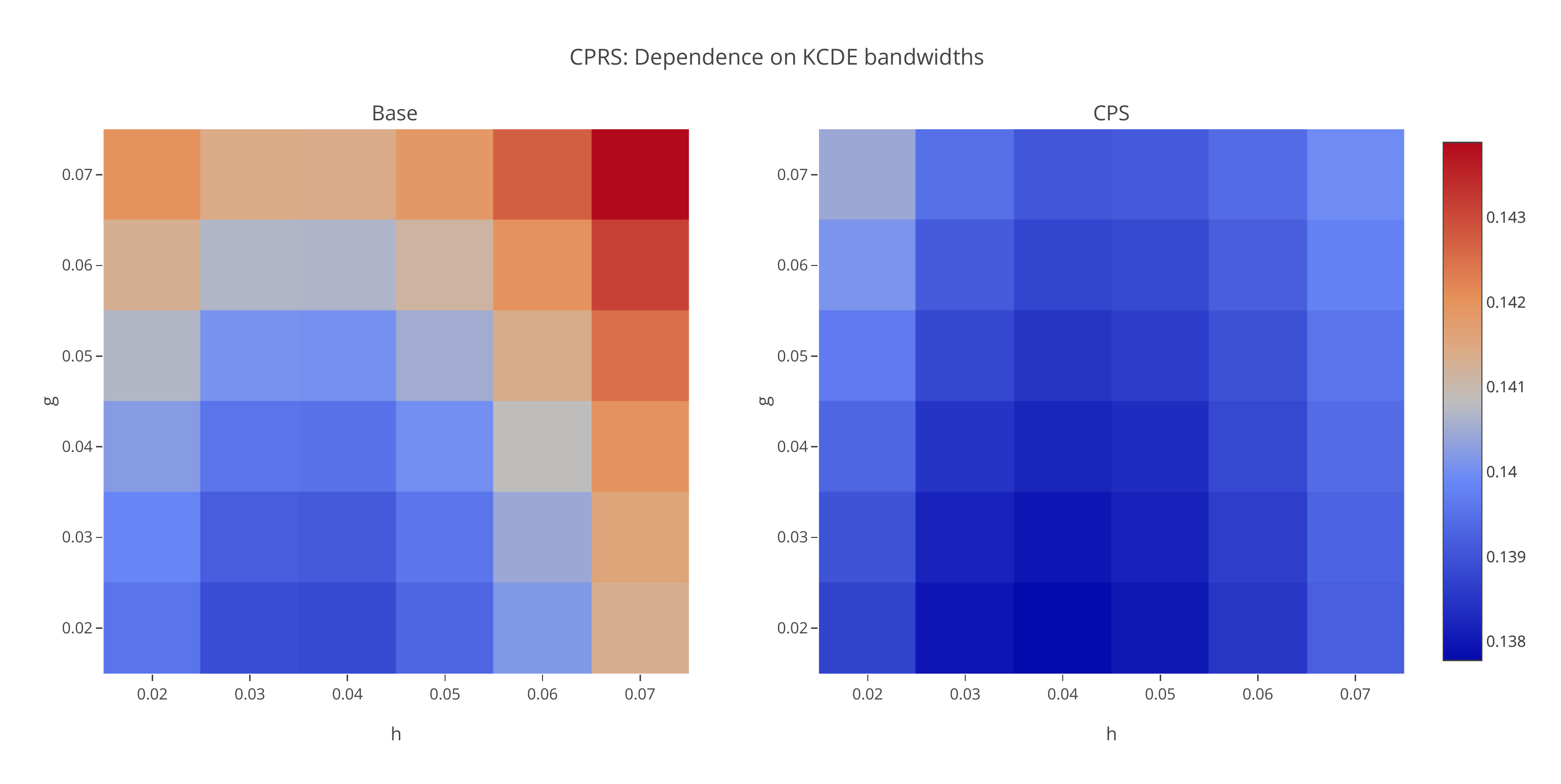}
  \end{center}
  \caption{Performance of the Nadaraya--Watson predictive system (left) and its conformalized version (right)
    for a range of $g$ and $h$.}
  \label{fig:heatmap}
\end{figure}

The loss of predictions is measured by CRPS (continuous ranked probability score),
as described in \cite[Section~7]{\OCMXXII}.
The left panel of Figure~\ref{fig:heatmap}
shows the loss,
averaged over $1000$ test observations,
of the Nadaraya--Watson predictive system~\eqref{eq:NW}
for various values of parameters $g$ and $h$.
The right panel shows the loss of the Nadaraya--Watson predictive system
calibrated using a separate calibration sequence of size $1000$.
We can see that calibration improves the performance of the base predictive system
for a wide range of parameter values.

\section{Calibration without the IID assumption}

A standard assumption in conformal prediction is that the observations are generated in the IID fashion
(sometimes this assumption is slightly weakened to assuming an online compression model,
as in \citealt[Chapter~8]{Vovk/etal:2005book}).
Therefore, it is interesting that Proposition~\ref{prop:adapt}
continues to hold in the absence of this assumption.
Indeed, the proof only depends on the probability integral transforms $A_{X_i}(Y_i)$
being distributed uniformly on $[0,1]$ and independent,
which does not require the IID assumption.
This is a well-known fact going back to L\'evy \citet[Section~39]{Levy:1937},
who only assumed that the distribution functions $A_x$ are continuous
for all $x\in\mathbf{X}$.
(Modern papers usually refer to Rosenblatt \citealt{Rosenblatt:1952},
who disentangled L\'evy's argument from his concern with the foundations of probability,
but Rosenblatt referred to L\'evy \citealt{Levy:1937} in his paper.)

To see an example where the conformalization procedure
works very well in the absence of the IID assumption,
suppose the base PS outputs the predictive distribution $\phi(A_x)$
for each test object $x$,
where $A$ is the true conditional distribution function
(defined by \eqref{eq:cdf})
and $\phi:[0,1]\to[0,1]$ is a very non-linear increasing function,
such as $\phi(u):=u^2$.
(So that the base PS has perfect resolution
but is badly miscalibrated.)
By Corollary~\ref{cor:adapt},
the conformalized version (not depending on $\phi$)
of the base PS will quickly converge to $A$,
whereas the base PS will always remain poor.

On the negative side,
in the absence of the IID assumption conformal calibrators
have no validity guarantees.

\section{Conclusion}

There are many directions of further research,
including:
\begin{itemize}
\item
  applying conformal calibrators to a wider range of artificial data and to benchmark datasets;
\item
  analyzing the predictive performance of conformal calibrators
  conditional on the test object $x$;
  optimizing conditional performance might require using Mondrian (namely, object-conditional) conformal calibrators
  and their modifications;
\item
  analyzing the predictive performance of conformal calibrators
  when applied to benchmark time series
  and in other non-IID situations.
\end{itemize}

\subsection*{Acknowledgments}

Thanks to Claus Bendtsen and the rest of the AstraZeneca team
and to Philip Dawid for useful discussions.
This work has been supported by
AstraZeneca (grant number R10911, ``Machine Learning for Chemical Synthesis'')
and Centrica.

\end{document}